\documentclass[letterpaper]{article}

\usepackage{aaai}

\usepackage{times}
\usepackage{helvet} 
\usepackage{courier} 
\usepackage[utf8]{inputenc} 
\usepackage[T1]{fontenc}    
\usepackage{url}            
\usepackage{booktabs}       
\usepackage{amsfonts}       
\usepackage{nicefrac}       
\usepackage{microtype}      
\usepackage{amsmath, amsthm, amssymb}
\usepackage{enumerate}

\nocopyright
\newcommand{\citet}[1]{\citeauthor{#1}~(\citeyear{#1})}
\title{Learning Generalized Reactive Policies\\ using Deep Neural Networks}

\author{
  Edward Groshev \\
  Department of Computer Science\\
  University of California, Berkeley\\
  Berkeley, CA 94720 \\
  \texttt{eddiegroshev@berkeley.edu} \\
  \And
  Maxwell Goldstein\\
  Department of Mathematics\\
  Princeton University\\
  Princeton, NJ 08544 \\
  \texttt{mag4@princeton.edu} \\
  \And
  Aviv Tamar\\
  Department of Computer Science\\
  University of California, Berkeley\\
  Berkeley, CA 94720 \\
  \texttt{avivt@berkeley.edu} \\
  \AND
  Siddharth Srivastava\thanks{Some of the work was done while this author was at United Technologies Research Center.}\\
  School of Computing, Informatics, \\and Decision Systems Engineering\\
  Arizona State University \\
  Tempe, AZ 85281 \\
  \texttt{siddharths@asu.edu} \\
  \And
  Pieter Abbeel\\
  Department of Computer Science\\
  University of California, Berkeley\\
  Berkeley, CA 94720 \\
  \texttt{pabbeel@cs.berkeley.edu} \\
}
\usepackage{multicol}
\usepackage{color}
\usepackage{amssymb}
\usepackage{amsmath}
\usepackage{dirtytalk}

\usepackage{booktabs}
\usepackage{graphicx}

\usepackage{caption}
\usepackage{subcaption}

\usepackage{algorithm}
\usepackage{algpseudocode}

\usepackage{graphicx}
\graphicspath{ {figures/} }

{\begin{list}{$\bullet$}{%
    \setlength{\topsep}{0in}
    \setlength{\partopsep}{0in}
    \setlength{\itemsep}{0in}
    \setlength{\parsep}{0in}
    \setlength{\leftmargin}{1.5em}
    \setlength{\rightmargin}{0in}
    \setlength{\itemindent}{-.1in}
}
}%
{\end{list}
}

\newtheorem{prop}{Proposition}
\renewcommand{\paragraph}[1]{\vspace{2pt}\newline\noindent\textbf{#1}\hspace{10pt}}

\newcommand{\E}{\mathcal{E}}
\newcommand{\F}{\mathcal{F}}
\newcommand{\R}{\mathcal{R}}
\newcommand{\Ops}{\mathcal{A}}

\newcommand{\Ntrain}{N_{\text{train}}}
\newcommand{\Ntest}{N_{\text{test}}}
\newcommand{\Dimit}{D_{\text{imitation}}}
\newcommand{\Dboot}{D_{\text{bootstrap}}}
\newcommand{\Dtrain}{D_{\text{train}}}
\newcommand{\Dtest}{D_{\text{test}}}
\newcommand{\nboot}{n_{\text{bootstrap}}}

\begin{document}

\maketitle

\begin{abstract}
We present a new approach to learning for planning, where knowledge acquired while solving a given set of planning problems is used to plan faster in related, but new problem instances. We show that a deep neural network can be used to learn and represent a \emph{generalized reactive policy} (GRP) that maps a problem instance and a state to an action, and that the learned GRPs efficiently solve large classes of challenging problem instances. In contrast to prior efforts in this direction, our approach significantly reduces the dependence of learning on handcrafted domain knowledge or feature selection. Instead, the GRP is trained from scratch using a set of successful execution traces. We show that our approach can also   be used to automatically learn a heuristic function that can be used in directed search algorithms. We evaluate our approach using an extensive suite of experiments on two challenging planning problem domains and show that our approach facilitates learning complex decision making policies and powerful heuristic functions  with minimal human input.
Videos of our results are available at \url{goo.gl/Hpy4e3}.

\end{abstract}

\vspace{-1em}
\section{Introduction}

In order to help with day to day chores such as organizing a cabinet or arranging a dinner table, robots need to be able plan: to reason about the best course of action that could lead to a given objective. Unfortunately, planning is well known to be a challenging computational problem: plan-existence for deterministic, fully observable  environments is PSPACE-complete when expressed using rudimentary propositional representations~\cite{bylander_94}. Such results have inspired multiple approaches for reusing knowledge acquired while planning across multiple problem instances (in the form of triangle tables~\cite{fikes_72}, learning control knowledge for planning~\cite{yoon_08}, and constructing generalized plans that solve multiple problem instances~\cite{srivastava11_genplan,hu11_genplan} with the goal of faster plan computation on a new problem instance.

In this work, we present an approach that unifies the principles of imitation learning (IL) and generalized planning for learning a \emph{generalized reactive policy} (GRP) that predicts the action to be taken, given an observation of the planning problem instance and the current state. The GRP is represented as a deep neural network (DNN). We use an off-the-shelf planner to plan on a set of training problems, and train the DNN to learn a GRP that imitates and generalizes the behavior generated by the planner. We then evaluate the learned GRP on a set of unseen test problems from the same domain. We show that the learned GRP successfully generalizes to unseen problem instances including those with larger state spaces than were available in the training set. This allows our approach to be used in end-to-end systems that learn representations as well as executable behavior purely from observations of successful executions in similar problems. 

We also show that our approach can generate representation-independent heuristic functions for a given domain, to be used in arbitrary directed search algorithms such as A$^*$~\cite{hart68_astar}. Our approach can be used in this fashion when stronger guarantees of completeness and classical notions of ``explainability'' are desired. Furthermore, in a process that we call ``leapfrogging", such heuristic functions can be used in tandem with directed search algorithms to generate training data for much larger problem instances, which in turn can be used for training more general GRPs. This process can be repeated, leading to GRPs that solve larger and more difficult problem instances with iteration.

While recent work on DNNs has illustrated their utility as function representations in situations where the input data can be expressed in an image-based representation, we show that DNNs can also be effective for learning and representing GRPs in a broader class of problems where the input is expressed using a graph data structure. For the purpose of this paper, we restrict our attention to deterministic, fully observable planning problems. We evaluate our approach on two planning domains that feature different forms of input representations. The first domain is Sokoban (see Figure \ref{fig:sokoban}). This domain represents problems where the execution of a plan can be accurately expressed as a sequence of images. This category captures a number of problems of interest in household robotics including setting the dinner table. This problem has been described as the most challenging problem in the literature on learning for planning \cite{fern2011first}.

Our second test domain is the traveling salesperson problem (TSP), which represents a category of problems where execution is \emph{not} efficiently describable through a sequence of images. This problem is challenging for classical planners as valid solutions need to satisfy a plan-wide property (namely a Hamiltonian cycle, which does not revisit any nodes). Our experiments with the TSP show that using graph convolutions~\cite{dai2017learning} DNNs can be used  effectively as function representations for GRPs in problems where the grounded planning domain is expressed as a graph data structure.  

Our experiments reveal that several architectural components are required to learn GRPs in the form of DNNs:
    (1) A \emph{deep} network.
    (2) Structuring the network to receive as input pairs of current state and goal observations. This allows us to `bootstrap' the data, by training with \emph{all pairs} of states in a demonstration trajectory.
    (3) Predicting plan length as an auxiliary training signal can improve IL performance. In addition, the plan length can be effectively exploited as a heuristic by standard planners.

We believe that these observations are general, and will hold for many domains. For the particular case of Sokoban, using these insights, we were able to demonstrate a 97\% success rate in one object domains, and an 87\% success rate in two object domains. In Figure \ref{fig:sokoban} we show an example test domain, and a non-trivial solution produced by our learned DNN.

\begin{figure*}[h]
\centering
\includegraphics[width=\textwidth]{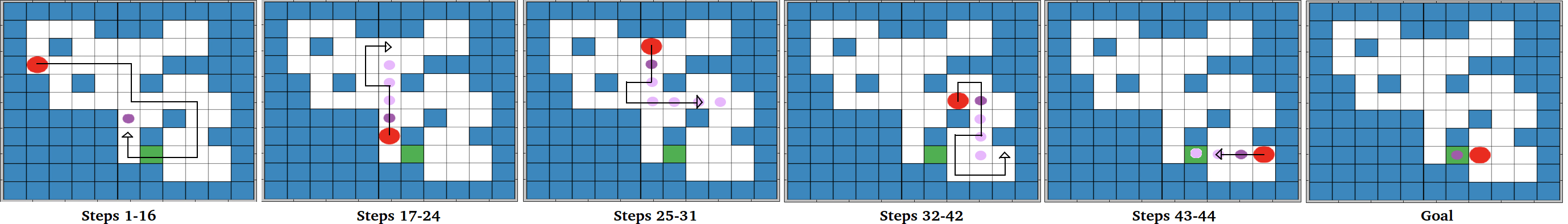}
\caption{The Sokoban domain (best viewed in color). In Sokoban the agent (red dot) must push around movable objects (purple dots) between unmovable obstacles (blue squares) to a goal position (green square). In this figure we show a challenging Sokoban instance with one object. From left to right, we plot several steps in the shortest plan for this task: arrows represent the agent's path, and light purple dots show the resulting object movement. This 44 step trajectory was produced by our learned DNN policy. Note that it demonstrates reasoning about dead ends that may happen many steps after the initial state.}\label{fig:sokoban}
\vspace{-1em}
\end{figure*}

\subsection{Related Work}
The interface of planning and learning~\cite{fern2011first} has been investigated extensively in the past. The works of  \citet{KHARDON1999125}, \citet{martin2000learning}, and \citet{yoon2002inductive} learn policies represented as decision lists on the logical problem representation, which must be hand specified. \citet{abel_icaps2015} learn action priors to prune the action space during planning. On the other hand, the literature on generalized planning~\cite{srivastava11_genplan,hu11_genplan} has focused on computing iterative generalized plans that solve broad classes of problem instances, with strong formal guarantees of correctness. In the reinforcement learning literature, \citet{konidaris_icml2006}, \citet{konidaris_jmlr2012}, and \citet{rosman_icdl2012} learn a shaping function and action priors respectively, to guides reinforcement learning on larger domains. \citet{torrey_ecml2006,torrey_cilp2007} approach skill transfer via inductive logic programming, where skills are manually crafted. While all of these strive to reuse knowledge obtained during planning, the selection of a good \emph{representation} for expressing the data as well as the learned functions or generalized plans is handcrafted. Feature sets and domain descriptions in these approaches are specified by experts using formal languages such as PDDL~\cite{fox_03}. Similarly, approaches such as case-based planning~\cite{spalzzi2001survey}, approaches for extracting macro actions~\cite{fikes_72,scala2015deordering} and for explanation based plan generalization~\cite{shavlik1989acquiring,kambhampati1994unified} rely on curated vocabularies and domain knowledge for representing the appropriate concepts necessary for efficient generalization of observations and the instantiation of learned knowledge. Our approach requires as input only a set of successful plans and their executions---our neural network architecture is able to learn a reactive policy that predicts the best action to execute based on the current state of the environment without any additional representational expressions. The current state is expressed either as an image (Sokoban) or as an instance of the graph data structure (TSP).

Neural networks have previously been used for learning heuristic functions~\cite{ernandes2004likely}. Recently, deep convolutional neural networks (DNNs) have been used to automatically extract expressive features from data, leading to state-of-the-art learning results in image classification \cite{krizhevsky2012imagenet}, natural language processing \cite{sutskever2014sequence}, and control \cite{mnih2015human}, among other domains. The phenomenal success of DNNs for across various disciplines motivates us to investigate whether DNNs can learn useful representations in the learning for planning setting as well. Indeed, one of the contributions of our work is a general convolutional DNN architecture that is suitable for learning to plan.

Imitation learning has been previously used with DNNs to learn policies for tasks that involve short horizon reasoning such as path following and obstacle avoidance \cite{pomerleau1989alvinn,ross2011reduction,tamar2016value,pfeiffer2016perception}, focused robot skills \cite{mulling2013learning,nair2017combining}, and recently block stacking \cite{duan2017one}. From a planning perspective, the Sokoban domain considered here is considerably more challenging than block stacking or navigation between obstacles. In value iteration networks \cite{tamar2016value}, a value iteration planning computation was embedded within the network structure, and demonstrated successful learning on 2D gridworld navigation. Due to the curse of dimensionality, it is not clear how to extend that work to planning domains with much larger state spaces, such as the Sokoban domain considered here. 
Concurrently with our work, \citet{weber2017imagination} proposed a DNN architecture that combines model based planning with model free components for reinforcement learning, and demonstrated results on the Sokoban domain. In comparison, our IL approach requires significantly less training instances of the planning problem (over 3 orders of magnitude) to achieve similar performance in Sokoban.

The `one-shot' techniques \cite{duan2017one}, however, are complimentary to this work. The impressive Alpha-Go-Zero~\cite{silver2017mastering} program learned a DNN policy for Go using reinforcement learning and self play. Key to its success is the natural curriculum in self play, which allows reinforcement learning to gradually explore more complicated strategies. A similar self-play strategy was essential for Tesauro's earlier Backgammon agent~\cite{tesauro1995temporal}. For the goal-directed planning problems we consider here, it is not clear how to develop such a curriculum strategy, although our leapfrogging idea takes a step in that direction.
Extending our work to reinforcement learning is a direction for future research.

Our approach thus offers two major advantages over prior efforts: (1) in situations where successful plan executions can be observed, e.g. by observing humans solving problems, our approach reduces the effort required in designing domain representations;  (2) in situations where guarantees of success are required, and domain representations are available, our approach provides an avenue for automatically generating a representation-independent heuristic function, which can be used with arbitrary guided search algorithms.

\section{Formal Framework}

We assume the reader is familiar with the formalization of deterministic, fully observable planning domains and planning problems in languages such as PDDL~\cite{fox_03,helmert09_pddl_grounding} and present the most relevant concepts here. A planning problem domain can be defined as a tuple $K=\langle \R, \Ops \rangle$, where $\R$ is  a set of binary \emph{relations}; and $\Ops$ is a set of \emph{parameterized actions}. Each action in $\Ops$ is defined by a set of preconditions categorizing the states on which it can be applied, and the set of instantiated relations that will changed to true or false as a result of executing that action.  A planning problem instance associated with a planning domain can be defined as $\Pi=\langle \E, s_0, G\rangle$, where $\E$ is a set of entities, $s_0$ is an initial state and $G$ is a set of goal conditions. Relations in $\R$ instantiated with  entities from $\E$ define the set of \emph{grounded fluents}, $\F$. Similarly, actions in $\Ops$ instantiated with appropriately entities in $\E$ define the set of \emph{grounded actions}, denoted as $\Ops[\E]$.  The initial state, $s_0$, for a given planning problem is a complete truth valuation of fluents in $\F$; the goal, $G$, is a truth valuation of a subset of the grounded fluents in $\F$. 

As an example, the discrete move action could be represented as follows: 
$$
{\emph{Move}}\text{(loc1, loc2)}:\begin{cases}
{\emph{pre}}:{RobotAt}(loc1), \\
{\emph{eff}}: \lnot{RobotAt}\text{(loc1)}, {RobotAt}\text{(loc2)} .
  \end{cases}
$$

We introduce several additional notations to the planning problem, to make the connection with imitation learning clearer. Given a planning domain and a planning problem instance, we denote by $S = 2^\F$ the state space of the planning problem. A state $s \in S$ corresponds to the values of each fluent in $\F$.  The task in planning is to find a sequence of grounded actions, $a_0, \ldots, a_n$ -- the so called \emph{plan} -- such that $a_n(\ldots(a_0(s_0))\ldots) \models G$. 

In Sokoban,  the domain represents the legal movement actions and the notion of movement on a bounded grid, a problem instance represents the exact grid layout (denoting which cell-entities are blocked), the starting locations of the objects and the agent, and the goal locations of the objects. 

We denote by $o(\Pi, s)$ the \emph{observation} for a problem  instance $\Pi$ when the state is $s$. For example, $o$ can be an image of the current game state (Figure \ref{fig:sokoban}) for Sokoban. We let $\tau = \left\{s_0, o_0, a_0, s_1, \dots, s_g, o_g\right\}$ denote the state-observation-action trajectory implied by the plan. The plan length is the number of states in $\tau$. 

Our objective is to learn a generalized behavior representation that efficiently solves multiple problem instances for a domain. More precisely, given a domain $K$, and a problem instance $\Pi$, let $\mathcal{O}_{K,\Pi}$ be the set of possible observations of states from $\Pi$. Given a planning problem domain $K =\langle \R, \Ops \rangle$ we define a \emph{generalized reactive policy (GRP)} as a function mapping observations of problem instances and states to actions: $GRP_K: \cup_\Pi\{\mathcal{O}_{K,\Pi}\} \rightarrow \cup_\Pi\{\Ops[\E_\Pi]\}$, where $\E_\Pi$ is the set of entities defined by the problem $\Pi$ and the unions range over all possible problem instances associated with $K$. Further, $GRP_K$ is constrained so that the observations from every problem instance are mapped to the grounded actions for that problem instance ($\forall \Pi \quad GRP_K(\mathcal{O}_{K, \Pi})\subseteq \Ops[\E_\Pi]$).
This effectively generalizes the concept of a policy to functions that can map states from multiple problem instances of a domain to action spaces that are legal within those instances.
\paragraph{Imitation Learning}
In imitation learning (IL), demonstrations of an expert solving a problem are given in the form of observation-action trajectories $\Dimit=\left\{o_0, a_0, o_1, \dots, o_T, a_T\right\}$. The goal is to find a policy -- a mapping from observation to actions $a = \mu(o)$, which imitates the expert. A straightforward IL approach is \emph{behavioral cloning} \cite{pomerleau1989alvinn}, in which  supervised learning is used to learn $\mu$ from the data.

\section{Learning Generalized Reactive Policies}
We assume we are given a set $\Dtrain$ of $\Ntrain$ problem instances $\left\{ \Pi_1,\dots,\Pi_{\Ntrain} \right\}$, which will be used for learning a GRP, and a set $\Dtest$ of $\Ntest$ problem instances that will be used for evaluating the learned model. We also assume that the training and test problem instances are similar in some sense, so that relevant knowledge can be extracted from the training set to improve performance on the test set. Concretely, both training and test  instances come from the same distribution. 

Our approach consists of two stages: a data generation stage and a policy training stage. 
\paragraph{Data generation} 
We generate a random set of problem instances $\Dtrain$. For each $\Pi\in \Dtrain$, we run an off-the-shelf planner to generate a plan and corresponding trajectory $\tau$, and then add the observations and actions in $\tau$ to $\Dimit$. In our experiments we used the Fast-Forward (FF) planner~\cite{FF}, though any other PDDL planner can be used instead.
\paragraph{Policy training}
Given the generated data $\Dimit$, we use IL to learn a GRP $\mu$. The learned policy $\mu$ maps an observation to action, and therefore can be readily deployed to any test problem in $\Dtest$. %

One may wonder why such a naive approach would even learn to produce the complex decision making ability that is required to solve unseen instances in $\Dtest$. Indeed, as we show in our experiments, naive behavioral cloning with standard shallow neural networks fails on this task. One of the contributions of this work is the investigation of DNN representations that make this simple approach succeed.

\subsection{Data Bootstrapping}
\label{ss:bootstrap}

In the IL literature  (e.g., \cite{pomerleau1989alvinn}), the policy is typically structured as a mapping from the observation of a state to an action. However, GRPs need to consider the problem instance while generating an action to be executed since different problem instances may have different goals. Although this seems to require more data, we present an approach for ``data bootstrapping'' that mitigates the data requirements.

Recall that our training data $\Dimit$ consists of $\Ntrain$ trajectories composed of observation-action pairs. This means that the number of training samples for a policy mapping state-observations to actions is equal to the number of observation-action pairs in the training data. However, since GRPs use the goal condition in their inputs (captured by a problem instance),\emph{any pair} of observations from successive states ($o(\Pi, s_i), o(\Pi, s_j)$) and the intermediate trajectory in an execution in $\Dtrain$ can be used as a sample for training the policy by setting $s_j$ as a goal condition for the intermediate trajectory. Our reasoning for this data bootstrapping technique is based on the following fact:

\begin{prop}
\label{prop:bootstrap}
For a planning problem $\Pi$ with initial state $s_0$ and goal state $s_g$, let $\tau_{opt} = \left\{ s_0, s_1, \dots, s_g\right\}$ denote the shortest plan from $s_0$ to $s_g$. Let $\mu_{opt}(s)$ denote an optimal policy for $\Pi$ in the sense that executing it from $s_0$ generates the shortest path $\tau_{opt}$ to $s_g$. Then, $\mu_{opt}$ is also optimal for a problem $\Pi$ with the initial and goal states replaced with any two states $s_i, s_j\in \tau_{opt}$ such that $i<j$.
\end{prop}

Proposition \ref{prop:bootstrap} underlies classical planning methods such as triangle tables~\cite{fikes_72}. Here, we exploit it to design our DNN to take as input \emph{both} the \emph{current observation} and a \emph{goal observation}. For a given trajectory of length $T$, the bootstrap can potentially increase the number of training samples from $T$ to $(T-1)^2/2$.
In practice, for each trajectory $\tau \in \Dimit$, we uniformly sample $\nboot$ pairs of observations from $\tau$. In each pair, the first observation is treated as the current observation, while the last observation is treated as the goal observation.\footnote{In our experiments, we used the FF planner, which does not necessarily produce shortest plans. However, Proposition \ref{prop:bootstrap} can be extended to satisficing plans.} This results in $\nboot+T$ training samples for each trajectory $\tau$, which are added to a bootstrap training set $\Dboot$ to be used instead of $\Dimit$ for training the policy.\footnote{Note that for the Sokoban domain, goal observations in the test set (i.e., real goals) do not contain the robot position, while the goal observations in the bootstrap training set include the robot position. However, this inconsistency had no effect in practice, which we verified by explicitly removing the robot from the observation. }

\subsection{Network Structure}
\label{ss:network_structure}

We propose a general structure for a convolutional network that can learn a GRP. 

Our network is depicted in Figure \ref{fig:network}. The current state and goal state observations are passed through several layers of convolution which are shared between the action prediction network and the plan length prediction network. There are also skip connections from the input layer to to every convolution layer.

The shared representation is motivated by the fact that both the actions and the overall plan length are integral parts of a plan. Having knowledge of the actions makes it easy to determine plan length and vice versa, knowledge about the plan length can act as a template for determining the actions. The skip connections are motivated by the fact that several planning algorithms can be seen as applying a repeated computation, based on the planning domain, to a latent variable. For example, greedy search expands the current node based on the possible next states, which are encoded in the domain; value iteration is a repeated modification of the value given the reward and state transitions, which are also encoded in the domain. Since the network receives no other knowledge about the domain, other than what's present in the observation, we hypothesize that feeding the observation to every conv-net layer can facilitate the learning of similar planning computations. We note that in value iteration networks~\cite{tamar2016value}, similar skip connections were used in an explicit neural network implementation of value iteration.

For planning in graph domains, we propose to use graph convolutions, similar to the work of~\cite{dai2017learning}. The graph convolution can be seen as a generalization of an image convolution, where an image is simply a grid graph. Each node in the graph is represented by a feature vector, and linear operations are performed between a node and its neighbors, followed by a nonlinear activation. A detailed description is provided in the supplementary material.
For the TSP problem with $n$ nodes, we map a partial Hamiltonian path $P$ of the graph to a feature representation as follows. 
For each node, the features are represented as a $3$-dimensional binary vector. The first element is 1 if the node has been visited in $P$, the second element is 1 if it is the current location of the agent, and the third element is 1 if the node is the terminal node. For a Hamiltonian cycle the terminal node is the start node. The state is then represented as a collection of feature vectors, one for each node. 
In the TSP every Hamiltonian cycle is of length $n$, so predicting the plan length in this case is trivial, as we encode the number of visited cities in the feature matrix. Therefore, we omit the plan-length prediction part of the network.

\subsection{Generalization to Different Problem Sizes}
A primary challenge in learning for planning is finding representations that can generalize across different problem sizes. For example, we expect that a good policy for Sokoban should work well on the instances it was trained on, $9\times9$ domains for example, as well as on larger instances, such as $12\times12$ domains. A convolution-based architecture naturally addresses this challenge.

However, while the convolution layers can be applied to any image/graph size, the number of inputs to the fully connected layer is strictly tied to the problem size. This means that the network architecture described above is fixed to a particular grid dimension. To remove this dependency, we employ a trick used in fully convolutional networks~\cite{long2015fully}, and keep only a $k\times k$ window of the last convolution layer, centered around the current agent position. This modification makes our DNN applicable to any grid dimension. Note that since the window is applied \emph{after} the convolution layers, the receptive field can be much larger than $k\times k$. In particular, a value of $k=1$ worked well in our experiments. For the graph architectures, a similar trick is applied, where the decision at a particular node is a function of the convolution result of its neighbors, and the same convolution weights are used across different graph sizes.

\begin{figure*}[h]
\centering
\includegraphics[width=0.95\textwidth]{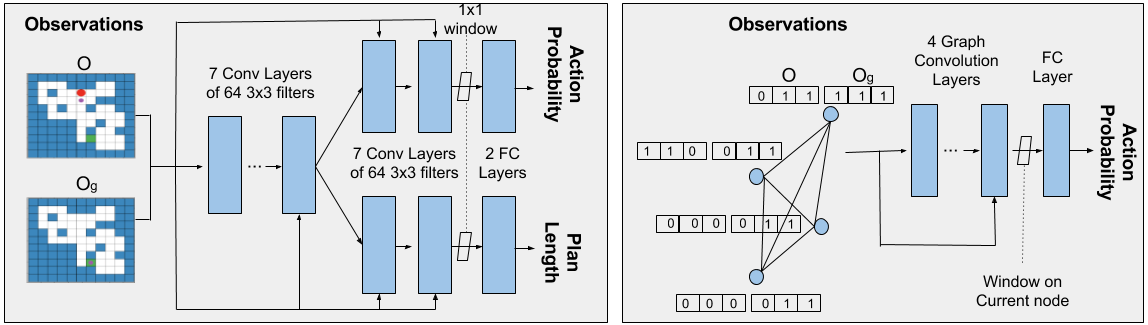}
\vspace{-0.5em}
\caption{Network architecture. The architecture on the left is used for Sokoban, while the one on the right is used for the TSP. A pair of current and goal observations are passed in to a shared conv-net. This shared representation is input to an action prediction conv-net and a plan length prediction conv-net. Skip connections from the input observations to all conv-layers are added. For the TSP network, we omitted the plan length prediction, as the features directly encode the number of nodes visited, making the prediction trivial. All activation functions are ReLU's and the final one is a SoftMax (multi-label classification used for action selection). In both architectures, after the last convolution layer, we apply a $k\times k$ window around the agents location to ensure a constant size feature vector is passed to the fully connected layers. This effectively decouples the architecture from the problem size and allows the receptive field to be greater than the $k\times k$ window.}\label{fig:network}
\vspace{-1em}
\end{figure*}

\section{Experiments}
Here we report our experiments\footnote{Sokoban datasets available at \url{https://github.com/edsterg/learning_grps} and TSP code available at \url{https://github.com/maxgold/generalized-gcn}} on learning for planning with DNNs. Our focus is on the following questions:
\begin{enumerate}
    \item What makes a good DNN architecture for learning a GRP?
    \item Can a useful planning heuristic be extracted from the GRP?
\end{enumerate}

The first question aims to show that recent developments in the representation learning community, such as deep convolutional architectures, can be beneficial for planning. The second question has immediate practical value -- a good heuristic can decrease planning costs. However, it also investigates a deeper premise. If a useful heuristic can indeed be extracted from the GRP, it means that the GRP has learned some underlying structure in the problem. In the domains we consider, such structure is hard to encode manually, suggesting that the data-driven DNN approach can be promising.

To investigate these questions, we selected two test domains representative of very different classes of planning problems. We used the \emph{Sokoban} domain to represent problems where plan execution can be captured as a set of images, and the goal takes the form of achieving a state property (objects at their target locations). We used the \emph{traveling salesperson problem} as an exemplar for problems where plan execution is not easy to capture as a set of images and the goal features a temporal property.
\paragraph{Sokoban} For Sokoban, we consider two difficulty levels: moving a single object as described in Figure \ref{fig:sokoban}, and a harder task of moving two objects. We generated training data using a Sokoban random level generator.\footnote{The Sokoban data-set from the learning for planning competition contains only 60 training domains, which is not enough to train a DNN. Our generator works as follows: we assume the room dimensions are a multiple of 3 and partition the grid into 3x3 blocks. Each block is filled with a randomly selected and randomly rotated pattern from a predefined set of 17 different patterns. To make sure the generated levels are not too easy and not impossible, we discard the ones containing open areas greater than 3x4 and discard the ones with disconnected floor tiles. For more details we refer the reader to Taylor et al.~\cite{Taylor2011SokobanGen}.}

For imitation learning, we represent the policy with the DNNs as described in Network Structure section 
and optimize using Adam~\cite{kingma2014adam}. When training with data bootstrapping, we selected $\nboot = T$ for generating $\Dboot$. Unless stated otherwise, the training set used in all Sokoban experiments was comprised of 45k observation-action trajectories (9k distinct obstacle configurations with 5 random start/goal locations per configuration).

To evaluate policy performance on the Sokoban domain we use execution success rate. Starting from the initial state, we execute the learned policy deterministically and track whether or not the goal state is reached.
We evaluate performance both on test domains of the same size the GRPs were trained on, $9\times 9$ grids, and also on larger problems. We explicitly verified that \emph{none of the test domains appeared in the training set}.

Videos of executions of our learned GRPs for Sokoban are available at \url{goo.gl/Hpy4e3}.
\paragraph{TSP} For TSP, we consider two different graph distributions. The first is the space of complete graphs with edge weights sampled uniformly in $[0,1]$. The second, which we term \emph{chord graphs}, is generated by first creating an $n$-node graph in the form of a cycle, and then adding $2n$ undirected chords between randomly chosen pairs of nodes, with a uniformly sampled weight in $[0,1]$. The resulting graphs are guaranteed to contain Hamiltonian cycles. However, in contrast to the complete graphs, finding such a Hamiltonian cycle is not trivial. Our results for the chord graphs are similar to the complete graphs, and for space constraints, we present them in the supplementary material. Training data was generated using the TSP solver in Google Optimization Tools\footnote{\url{https://developers.google.com/optimization}}.

As before, we train the DNN using Adam. We found it sufficient to use only 1k observation-action trajectories for our TSP domain. The metric used is average relative cost\footnote{For the complete graphs, all policies always succeeded in finding a Hamiltonian cycle. For the chord graphs, we report success rates in the supplementary material.}, defined as the ratio between the cycle cost of the learned policy and the Google solver, averaged over all initial nodes in each test domain. We also compare the DNN policy against a greedy policy which always picks the lowest-cost edge leading to an unvisited node. 

As in the Sokoban domain, we evaluate performance on test domains with graphs of the same size as the training set, 4 node graphs, and on larger graphs with up-to 11 nodes.

\subsection{Evaluation of Learned GRPs}
Here we evaluate performance of the learned GRPs on previously unseen test problems. Our results suggest that the GRP can learn a well-performing planning-like policy for challenging problems. In the Sokoban domain, on $9\times 9$ grids, the learned GRP in the best performing architecture (14 layers, with bootstrapping and a shared representation) can solve one-object Sokoban with 97\% success rate, and two-object Sokoban with 87\% success rate. Figure \ref{fig:sokoban} shows a trajectory that the policy predicted in a challenging one-object domain from the test set. Two-object trajectories are difficult to illustrate using images; we provide a video demonstration at \url{goo.gl/Hpy4e3}. We observed that the GRP effectively learned to select actions that avoid dead ends far in the future, as Figure~\ref{fig:sokoban} demonstrates. The most common failure mode is due to cycles in the policy, and is a consequence of using a deterministic policy. Further analysis of failure modes is given in the supplementary material. The learned GRP can thus be used to solve new planning problem instances with a high chance of success. In domains where simulators are available, a planner can be used as a fallback if the policy fails in simulation.

For TSP, Figure \ref{fig:tsp_results} shows the performance 
of the GRP policy on complete graphs of sizes $4-11$, when trained on graphs of the same size (respectively). For both the GRP and the greedy policy, the cost increases approximately linearly with the graph size. For the greedy policy, the rate of cost increase is roughly twice the rate for the GRP, showing that the GRP learned to perform some type of lookahead planning.

\subsection{Investigation of Network Structure}
We performed ablation experiments to tease out the important ingredients for a successful GRP. Our results suggest that deeper networks improve performance.

In Figure~\ref{fig:experiment1} we plot execution success rate on two-object Sokoban, for different network depths, and with or without skip connections. The results show that deeper networks perform better, with skip connections resulting in a consistent advantage. In the supplementary material we show that a deep network significantly outperformed a shallow network with the same number of parameters, further establishing this claim.
The performance levels off after 14 layers. We attribute this to the general difficulty of training deep DNNs due to gradient propagation, as evident in the failure of training the 14 layer architecture without skip connections.

\begin{figure*}[h]
\vspace{-1em}
  \centering
    \noindent
\hfill
  \begin{subfigure}[b]{0.31\linewidth}
    \includegraphics[width=\textwidth]{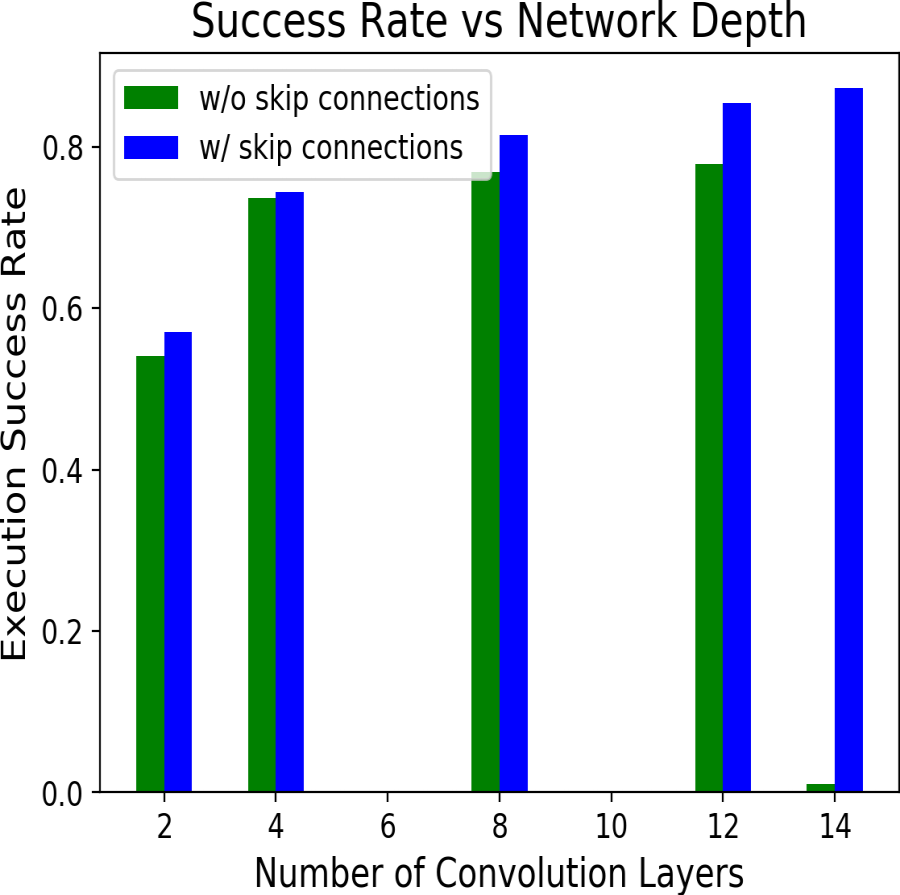}
    \caption{}
    \label{fig:experiment1}
  \end{subfigure}
\hfill
  \begin{subfigure}[b]{0.32\linewidth}
    \includegraphics[width=\textwidth]{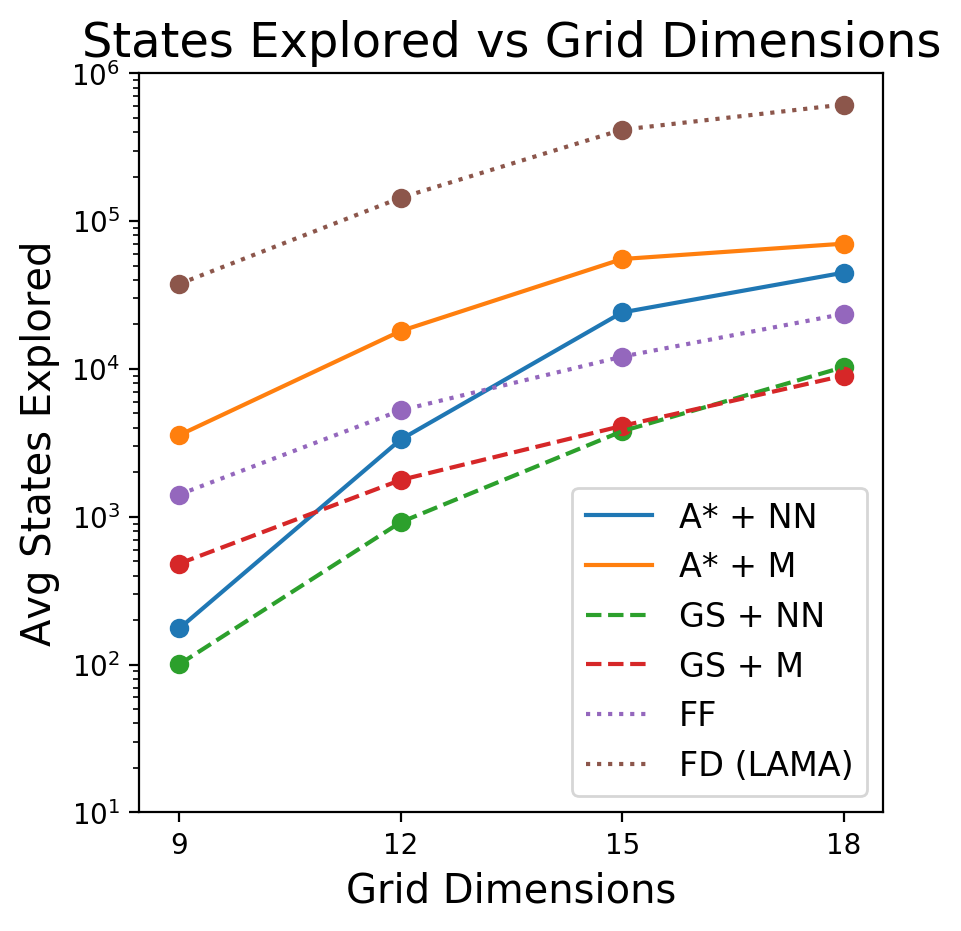}
    \caption{}
    \label{fig:heuristic}
  \end{subfigure}
  \hfill
  \begin{subfigure}[b]{0.3\linewidth}
    \includegraphics[width=\textwidth]{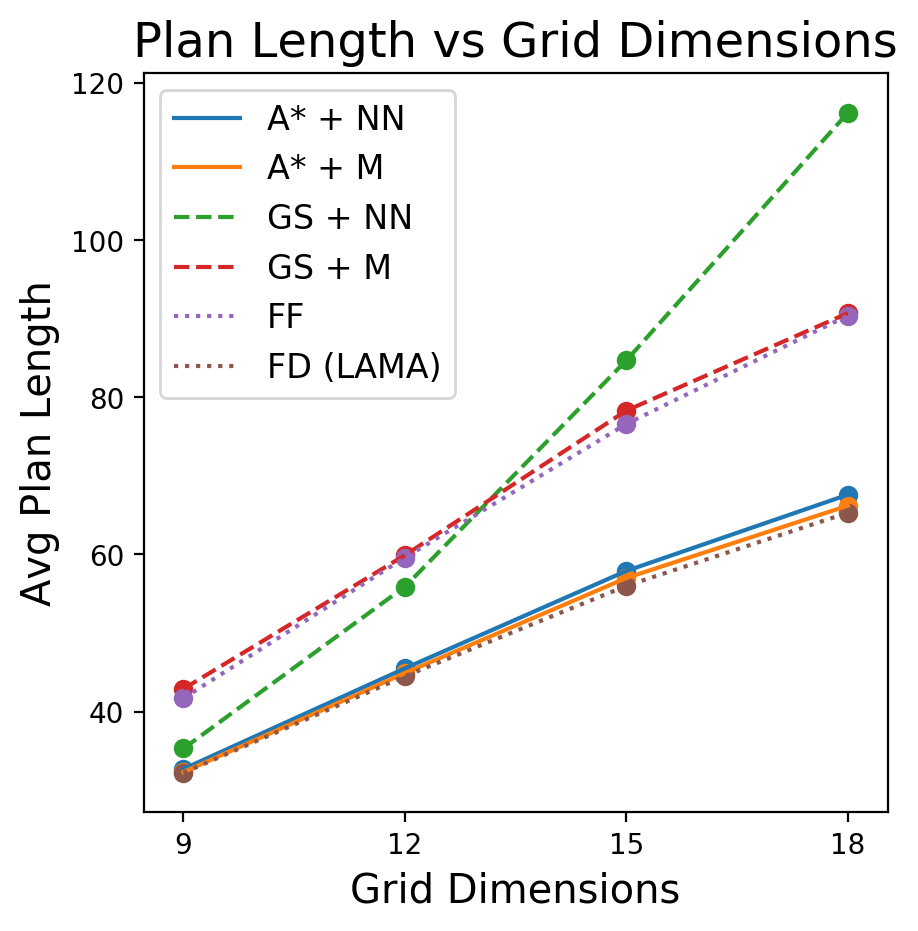}
    \caption{}
  \end{subfigure}
\vspace{-0.5em}
  \caption{Sokoban results. (a) Investigating DNN depth and skip connections. We plot the success rate for deterministic execution in two-object Sokoban. Deeper networks show improved success rates and skip connections improve performance consistently. We were unable to successfully train a 14 layer deep network without skip connections.
  (b,c) Performance of learned heuristic. The GRP was trained only on 9x9 instances, and evaluated (as a heuristic, see text for more details) on larger instances. (b) shows number of states explored (i.e., planning speed) and (c) shows plan length (i.e., planning quality). A* with the learned heuristic produced nearly optimal plans with an order of magnitude reduction in the number of states explored. All the differences in (b) are significant according to a Wilcoxon signed-rank test with significance $0.1\%$ and $p<1\times 10^{-6}$.}
  \label{fig:experiments}
\end{figure*}

\begin{figure*}[h!]
  \centering
    \hfill
    \begin{subfigure}[b]{0.32\linewidth}
    \includegraphics[width=\textwidth]{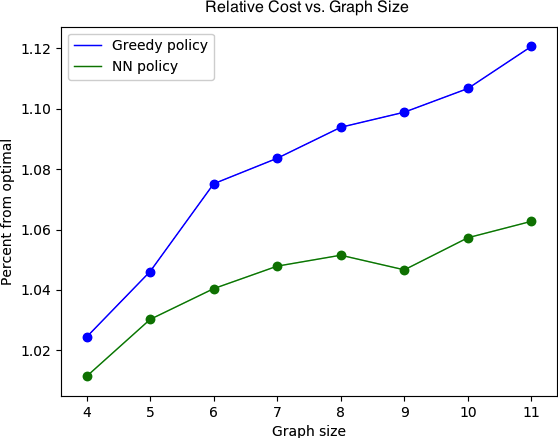}
    \caption{}
    \label{fig:tsp_results}
  \end{subfigure}\hfill
  \begin{subfigure}[b]{0.32\linewidth}
    \includegraphics[width=\textwidth]{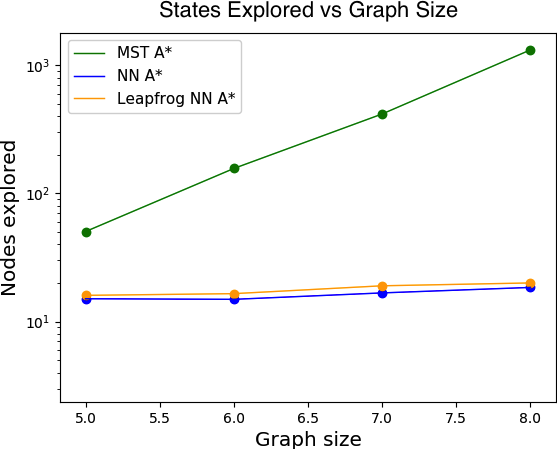}
    \caption{}
  \end{subfigure}\hfill
  \begin{subfigure}[b]{0.32\linewidth}
    \includegraphics[width=\textwidth]{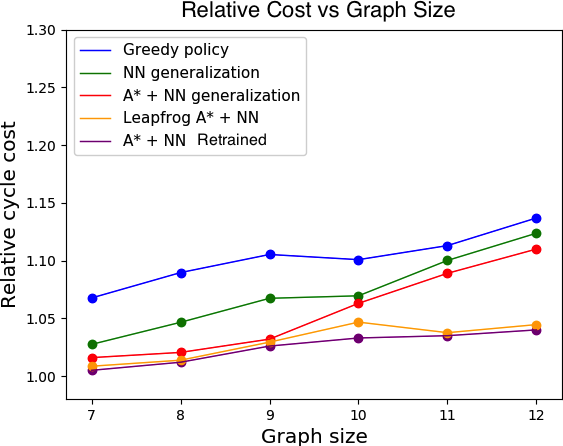}
    \caption{}
  \end{subfigure}
  \caption{TSP results. (a) Performance (average relative cost; see text for details) for GRPs trained and tested on problems of sizes $4-11$, respectively. We compare the GRP with a greedy policy. (b,c) Performance of learned heuristic. The GRP was trained on 4-node graphs, and evaluated (as a heuristic, see text for more details) on larger instances. (b) shows number of states explored (i.e., planning speed). We compare with the minimum spanning tree heuristic, which is admissible for TSP.  (c) shows average relative cost (i.e., planning quality) compared to plans from the Google solver. Note that up to a graph of size $9$, the performance of A$^*$ with GRP heuristic (labeled A$^*$+NN generalization) was within $5\%$ of optimal, while requiring orders of magnitude less computation than the MST heuristic.  We also present results for the leapfrogging algorithm (see text for details), and additionally compare to a baseline of retraining the GRP with optimal data for each graph size. Note that the leapfrogging results are very close to the results obtained with retraining, although optimal data was only given for the smallest graph size. This shows that the GRP heuristic can be used for generating reliable training data for domains of larger size than trained on.}
  \label{fig:tsp_execution}
  \vspace{-1em}
\end{figure*}

We also investigated the benefit of having a shared representation for both action and plan length prediction, compared to predicting each with a separate network. The ablation results are presented in Table \ref{table:bootstrap_compare}. Interestingly, the plan length prediction improves the accuracy of the action prediction.

\begin{table}[h]
\begin{center}
\resizebox{0.48 \textwidth}{!}{   
  \begin{tabular}{ c | c | c l }
        & w/ bootstrap & w/o bootstrap\\
    \hline
    Predict plan length & 2.211 & 2.481 & $\ell_1$ norm\\
    \hline
    Predict plan length  & \textbf{2.205} & 2.319 & $\ell_1$ norm\\ \cline{2-4}
    \& actions    & \textbf{0.844} & 0.818 & Succ Rate\\
    \hline
    Predict actions & 0.814 & 0.814 & Succ Rate\\
    \hline
  \end{tabular}\par
}

  \bigskip
  \vspace{-1em}
  \caption{Benefits of bootstrapping and having a shared representation. To evaluate accuracy of the plan length prediction, we measure the average $\ell_1$ loss (absolute difference). To evaluate action prediction we measure the success rate on execution. Best performance was obtained with using bootstrapping and the shared representation. For this experiment the training set contained 25k observation-action trajectories from 5k different obstacle configurations.
  }\label{table:bootstrap_compare}
  \vspace{-2em}
\end{center}
\end{table}

\subsection{GRP as a Heuristic Generator}

We now show that the learned GRPs can be used to extract \emph{representation independent heuristics} for use with arbitrary guided search algorithms. To our knowledge, there are no other approaches for computing such heuristics without using hand-curated domain vocabularies or features for learning and/or expressing them. However, to evaluate the quality of our learned heuristics, we compared them with a few well-known heuristics that are either handcrafted or computed using handcrafted representations. We found that the representation-independent GRP heuristic was competitive, and remains effective on larger problems than the GRP was trained on. 
For the Sokoban domain, the plan-length prediction can be directly used as a heuristic function. This approach can be used for  state-property based goals in problems where execution can be captured using images. For the TSP domain, we used a heuristic that is inversely proportional to the probability of selecting the next node to visit, as the number of steps required to create a complete cycle is not discriminative. Full details are given in the supplementary material.

We investigated using the GRP as a heuristic in greedy search and A$^*$ search \cite{hart68_astar}. We use two performance measures: the number of states explored during search and the length of the computed plan. The first measure corresponds to planning speed since evaluating less nodes translates to faster planning. The second measure represents plan quality. 
\paragraph{Sokoban} We compare performance in Sokoban to the Manhattan heuristic\footnote{The Manhattan heuristic is only admissible in one-object Sokoban. We tried Euclidean distance and Hamiltonian distance. However, Manhattan distance had the best trade-off between performance and computation time.} in Figure \ref{fig:heuristic}. In the same figure we evaluate generalization of the learned heuristic to larger, never before seen, instances as well as the performance of two state-of-the-art planners: LAMA~\cite{lama2011} which uses the Fast Downward (FD, \cite{helmert2006fast}) planning framework, and Fast Forward (FF, \cite{FF}) planner.\footnote{We constrained the anytime LAMA algorithm to 5 minutes per instance. For all instances we evaluated, LAMA always found the optimal solution.} The GRP was trained on $9 \times 9$ domains, and evaluated on new problem instances of similar size or larger. During training, we chose the window size $k=1$ to influence learning a problem-instance-size-invariant policy. As seen in Figure \ref{fig:heuristic} the learned GRP heuristic \emph{significantly outperforms the Manhattan heuristic} in both greedy search and A* search, on the 9x9 problems. As the size of the test problems increases, the learned heuristic shines when used in conjunction with A*, consistently expanding fewer nodes than the Manhattan heuristic. Note that even though the GRP heuristic is not guaranteed to be admissible, when used with A*, the plan quality is very close to optimal, while exploring an order of magnitude less nodes than the conventional alternatives.
\paragraph{TSP} 
We trained the GRP on 6-node complete graphs and evaluated the GRP, used either directly as a policy or as a heuristic within A$^*$, on graphs of larger size.
Figure \ref{fig:tsp_execution}(b-c) shows generalization performance of the GRP, both in terms of planning speed (number of nodes explored) and in terms of plan quality (average relative cost). We compare both to a greedy policy, and to A$^*$ with the minimum spanning tree (MST) heuristic. Note that the GRP heuristic is significantly more efficient than MST, while not losing much in terms of plan quality, especially when compared to the greedy policy.

\subsection{Leap-Frogging Algorithm}
The effective generalization of the GRP heuristic to larger problem sizes motivates a novel algorithmic idea for learning to plan on iteratively increasing problem sizes, which we term \emph{leap-frogging}. The idea is that, we can use a `general and optimal' planner, such as LAMA, to generate data for a small domain, of size $d$. We then train a GRP using this data, and use the resulting GRP heuristic in A$^*$ to \emph{quickly} solve planning problems from a larger domain $d'>d$. These solutions can then be used as new data for training another GRP on the domain size $d'$. Thus, we can iteratively apply this procedure to solve problems of larger and larger sizes, while only requiring the slow `general' planner to be applied in the smallest domain size. 

In Figure \ref{fig:tsp_execution}c we demonstrate this idea in the TSP domain. We used the solver to generate training data for a graph with 4 nodes. We then evaluate the GRP heuristic trained using leapfrogging on larger domains, and compare with a GRP heuristic that was only trained on the 4-node graph. 
Note that we significantly improve upon the standard GRP heuristic, while using the same initial optimal data obtained from the slow Google solver. We also compare with a GRP heuristic that was re-trained with optimal data for each graph size. Interestingly, this heuristic performed only slightly better than the GRP trained using leap-frogging,
showing that the generalization of the GRP heuristic is effective enough to produce reliable new training data.

\section{Conclusion}

We presented a new approach in learning for planning, based on imitation learning from  execution traces of a planner. We used deep convolutional neural networks for learning a generalized policy, and proposed several network designs that improve learning performance in this setting, and are capable of generalization across problem sizes. 
Our networks can be used to extract an effective heuristic for off-the-shelf planners, improving over standard heuristics that do not leverage learning.

Our results on the challenging Sokoban domain suggest that DNNs 
have the capability to 
extract powerful features from observations, 
and 
the potential to learn the type of `visual thinking' that makes some planning problems easy for humans but very hard for automatic planners. The leapfrogging results, suggest a new approach for planning -- when facing a large and difficult problem, first solve simpler instances of the problem and learn a DNN heuristic that aids search algorithms in solving larger instances. This heuristic can be used to generate data for training a new DNN heuristic for larger instances, and so on.  Our preliminary results suggest this approach to be promising.

There is still much to explore in employing deep networks for planning. While representations for images based on deep conv-nets have become standard, representations for other modalities such as graphs and logical expressions are an active research area~\cite{dai2017learning,kansky2017schema}. We believe that the results presented here will motivate future research in representation learning for planning.

\subsection*{Acknowledgement}
This research was supported in part by Siemens, Berkeley Deep Drive, and an ONR PECASE N000141612723.

\begin{small}
\makeatletter
\renewcommand\@biblabel[1]{}
\makeatother
\bibliographystyle{aaai}
\bibliography{references}
\end{small}

\newpage
\section{Appendix}
\subsection{Graph Convolution Network}
Consider a graph $\mathcal{G} = (V, \mathcal{E})$ with adjacency matrix $A$ where $V$ has $N$ nodes and $\mathcal{E}$ is the weighted edge set with weight matrix $W$.  Suppose that each node  $v \in V$ has a corresponding feature $x_v \in \mathbb{R}^m$ and each edge $(u,v) \in E$ corresponds to $e_{uv} \in \mathbb{R}^n$  and consider a parametric function $f_\theta : \mathbb{R}^{2m+n} \rightarrow \mathbb{R}^{\hat m}$ parameterized by $\theta \in \mathbb{R}^f$. Let $\mathcal{N}_i : V \rightarrow 2^V$ denote a function mapping a vertex to its $i$th degree neighborhood.
The propagation rule is given by the following equation
\begin{align}
H_{v} = \sigma\left(\sum_{u \in \mathcal{N}(v)} A_{uv} f_\theta(x_u, x_v, e_{uv})\right)
\end{align}
where $\sigma$ is the ReLU function.
Consider a graph $\mathcal{G}$ of size $n$, with each vertex having feature vector of size $C$ encoded in the feature matrix $X \in \mathbb{R}^{n\times C}$. In the TSP experiments, we use the propagation rule to generate $H\in \mathbb{R}^{n\times C'}$, where $C'$ is the number of features in the next layer (a.k.a. depth of the layer), and where the $ij$ entry of $H$ is given by
 \begin{align}
 H_{ij} = \sigma\left(\sum_{s \in \mathcal{N}(i)} A_{si} [x_s, x_i, W_{si}]^T \Theta_j + b_j\right)
\end{align}
Here, $W$ is the weight matrix of  $\mathcal{G}$, $A$ is the adjacency matrix, and $\Theta \in \mathbb{R}^{(2C+1)\times C'}$ is the matrix of weights that we learn and $b \in \R^{C'}$ is a learned bias vector. $\Theta_j$ is the $j$th column of $\Theta$.

In the networks we used for the TSP domain, the initial feature vector is of size $C=6$. We then applied $4$ convolution layers of size $C=26$.
We then applied a convolution of size $C=1$, corresponding to a fully connected layer. Thus, $j=1$ in $H_{ij}$ for all $i$ in the last convolution layer.

The final layer of the network is a softmax over $H_{i1}$, and we select the node $i$ with the highest score that is also connected to the current node.
\paragraph{Relation to Image Convolution} In the next proposition we show that this graph-based propagation rule can be seen as a generalization of a standard 2-D convolution, when applied to images (grid graphs). Namely, we show that there exists features for a grid graph and parameters $\Theta$ for which the above propagation rule reduces to a standard 2-D convolution. 
\begin{prop}
When $\mathcal{G}$ is a grid graph, for a particular choice of $f_\theta$ the above propagation rule reduces to the traditional convolutional network. In particular, for a filter of size $n$, choosing $f_\theta$ as a polynomial of degree $2(N-1)$ and $\theta \in \mathbb{R}^{N^2}$ works.
\end{prop}
\begin{proof}
For each node $v$, consider its representation as $v = (v_x, v_y)$ where $(v_x, v_y)$ are the grid coordinates of the vertex.

Let $a := \frac{n-1}{2}$. We first transform the coordinates to center them around $v$ by transforming $u \rightarrow (u_x - v_x, u_y - v_y)$ so that $u$ lies in the set $[-a, a] \times [-a,a]$. 

We wish to design a polynomial $g$ that takes the value $\theta_{i,j}$ at location $(i,j)$. We show that it is possible to do with a degree $2(n-1)$ polynomial by construction. The polynomial $g$ is given by
\begin{align}
g(x, y) := \sum_{i=-a}^{a} \sum_{j=-a}^a  \theta_{i,j} \prod_{s=-a, s\neq i}^a (s+y) \prod_{t=-a, t\neq j}^a (t+x)
\end{align}
To see why this is correct, note that for any $(s,t) \in [-a, a] \times [-a,a]$ there is exactly one polynomial inside the summands that does not have either of the terms $(i + u_y)$ or $(j + u_x)$ appearing in its factorization. Indeed, by construction this term is the polynomial corresponding to $\theta_{i,j}$ so that $g(i,j) = C \theta_{i,j}$ for some constant $C$. 

The polynomial inside the summands is of degree $(n-1)+(n-1) = 2(n-1)$, so $g$ is of degree $2(n-1)$. Letting $p_u$ denote th pixel value at node $u$, setting 
\begin{align}
f_\theta(x_u, x_v) := p_u g(x_u - x_v)
\end{align}
completes the proof.
\end{proof}

\subsection{TSP domain heuristic}
We can use the graph convolution network as a heuristic inside A-star search. Given a feature encoding of a partial cycle $P$, we can compute the probability $p_i$ of moving to any node $i$. We then use the quantity ${(N - v)(1-p_i)/2}$ as the heuristic, where $N$ is the total number of nodes and $v$ is the number of visited nodes in the current partial path. Multiplying by $(N-v)/2$ puts the output of the heuristic on the same scale as the current cost of the partial path.

\subsection{Deep VS Shallow Networks}
Here we present another experiment to further establish the claim that the depth of the network improves performance and not necessarily the number of parameters in the network. In Table \ref{table:deep_vs_shallow} we compare deep networks against shallow networks containing the same number of parameters. Note that we evaluate based on two different metrics. The first metric is classification error on the next action, which shows whether or not the action matches what the planner would have done. The second metrics is execution success rate, as defined above.

\begin{table}[h]
\begin{center}
\resizebox{0.45 \textwidth}{!}{   
  \begin{tabular}{ c | c | c | c c }
    Num Params & Deep-8 & Wide-2 & Wide-1\\
    \hline
    556288 & 0.068 & 0.092 & 0.129 & error rate\\ \cline{2-5}
     & 0.83 & 0.62 & 0.38 & succ rate\\
     \hline
  \end{tabular}\par
 }
  \bigskip
  \vspace{-1em}
  \caption{Comparison of deep vs. shallow networks. The deep network has 8 convolution layers with 64 filter per layer. The shallow networks contain 2 and 1 layers respectively with 256 and 512 filters per layer respectively. Clearly, deeper networks outperform shallow networks while containing an equal number of parameters.}\label{table:deep_vs_shallow}

\end{center}
\end{table}

\subsection{Evaluation of Bootstrap Performance}
\label{ss:bootstrap_exp}
We briefly summarize the evaluation of data bootstrapping in the Sokoban domain. Table \ref{table:bootstrap_compare} shows the success rate and plan length prediction error for architectures with and without the bootstrapping. As can be observed, the bootstrapping resulted in better use of the data, and led to improved results.

While investigating the performance of data bootstrapping with respect to training set size, we observed that a non-uniform sampling performed better on smaller datasets.
For each $\tau \in \Dimit$, we sampled an observation $\hat{o}$ from a distribution that is linearly increasing in time, such that observations near the goal have higher probability. The performance of this bootstrapping strategy is shown in Figure \ref{fig:small_bootstrapping}. As should be expected, performance improvement due to data augmentation is more significant for smaller data sets.

\begin{figure}[h]
\vspace{-1em}
  \centering
    \noindent
  \includegraphics[width=0.3\textwidth]{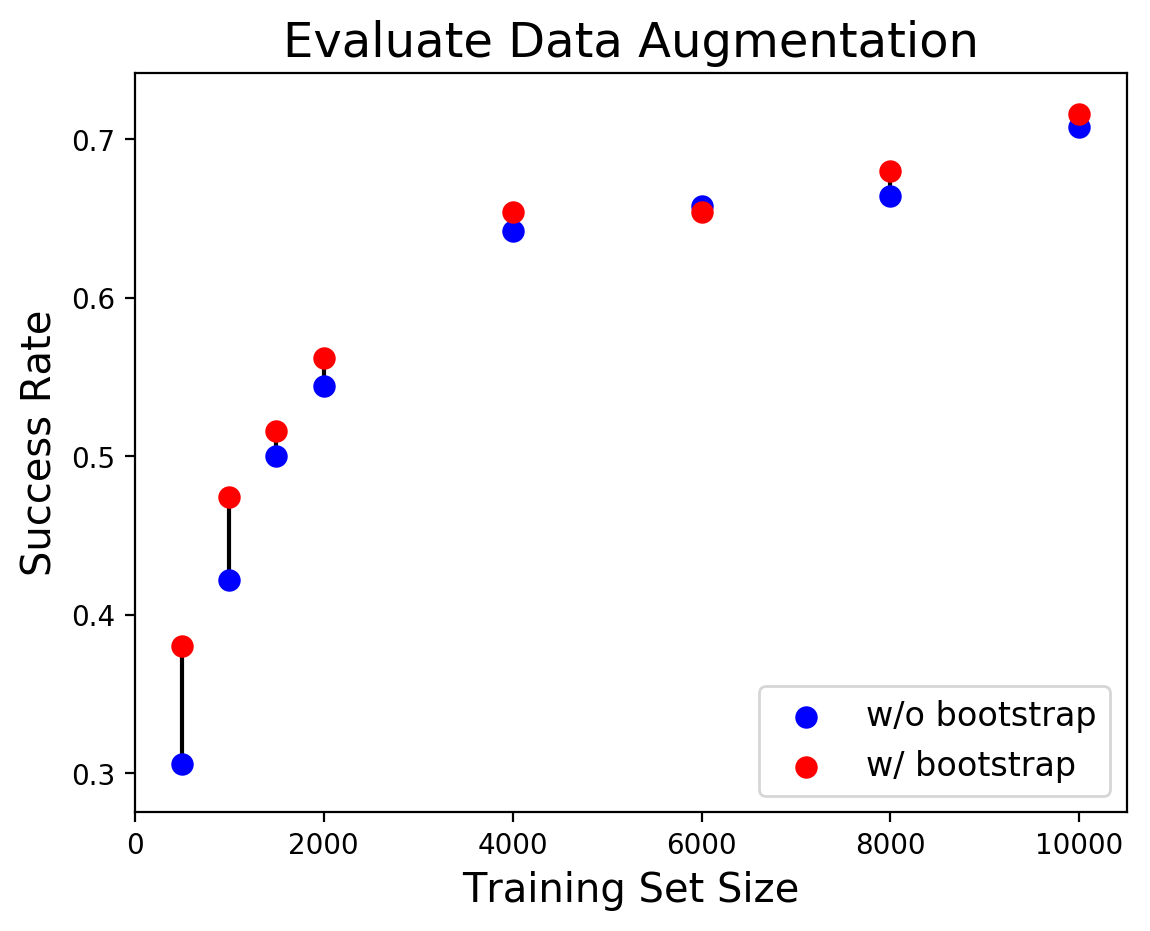}
  \caption{This shows the affect of data bootstrapping on the performance of two-object Sokoban, as a function of the dataset size. Smaller datasets benefit more from data augmentation.}
  \label{fig:small_bootstrapping}
\end{figure}

\begin{figure*}[h]
\vspace{-1em}
  \centering
    \noindent
    \begin{subfigure}[b]{0.27\linewidth}
    \includegraphics[width=\textwidth]{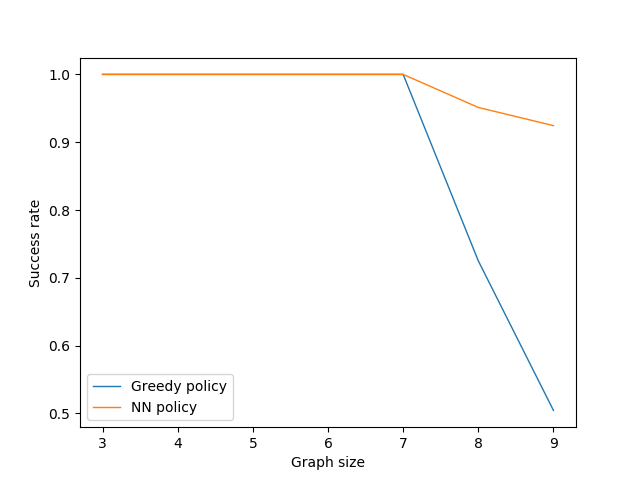}
    \caption{}
    \label{fig:tsp_results2}
  \end{subfigure}\hfill
  \begin{subfigure}[b]{0.27\linewidth}
    \includegraphics[width=\textwidth]{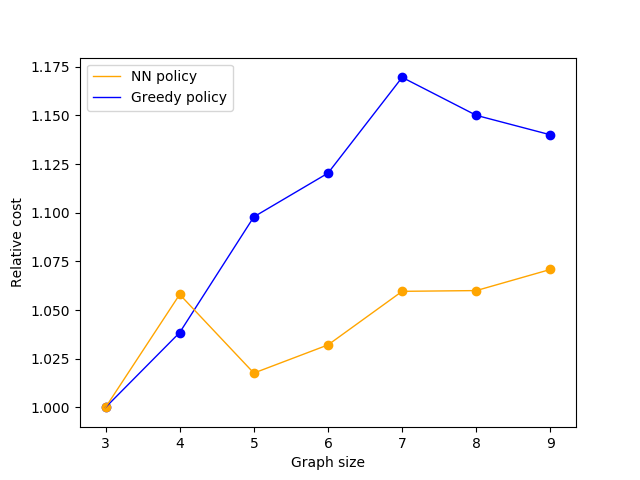}
    \caption{}
  \end{subfigure}\hfill
  \begin{subfigure}[b]{0.27\linewidth}
    \includegraphics[width=\textwidth]{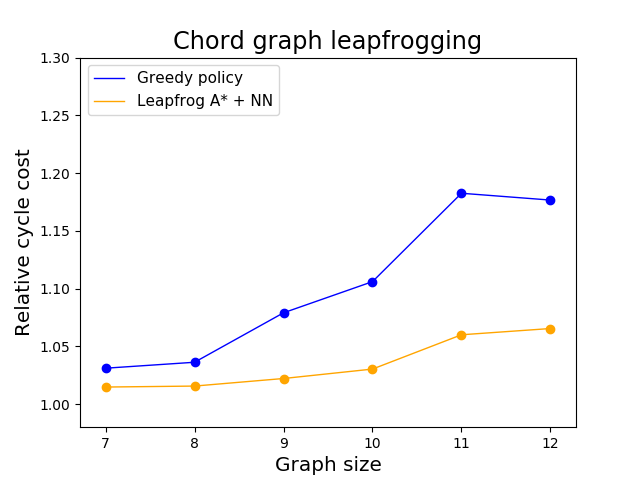}
    \caption{}
  \end{subfigure}\hfill
  \caption{Chord-graph TSP results. (a) Success rate of neural network policy on chord graphs of size $3-9$, respectively. Note that the agent is only allowed to visit each node once, so the agent may visit a node with no un-visited neighbors which is a dead end. We also show the success rate of the greedy policy. (b) Performance of neural network policy on chord graphs of size 3-9. (c) Leapfrogging algorithm results on chord graphs of size 7-12. We compare to a baseline greedy policy } 
  \label{fig:tsp_execution2}
\end{figure*}

\subsection{Analysis of Failure Modes}\label{ssec:failure_modes}

While investigating the failure modes of the learned GRP in the Sokoban domain, we noticed that there were two primary failure modes. The first failure mode is due to cycles in the policy, and is a consequence of using a deterministic policy. For example, when the agent is between two objects a deterministic policy may oscillate, moving back and fourth between the two. We found that a stochastic policy significantly reduces this type of failure. However, stochastic policies have some non-zero probability of choosing actions that lead to a dead end (e.g., pushing the box directly up against a wall), which can lead to different failures. The second failure mode was the inability of our policy to foresee long term dependencies between the two objects. An example of such a case is shown in Figure \ref{fig:failure_modes} (f-h), where deciding which object to move first requires a look-ahead of more than 20 steps. A possible explanation for this failure is that such scenarios are not frequent in the training data.
This is less a limitation of our approach and more a limitation of the neural network, more specifically the depth of the neural network.

Additionally, we investigated whether the failure cases can be related to specific features in the task. Specifically, we considered the task plan length (computed using FD), the number of walls in the domain, and the planning time with the FD planner (results are similar with other planners). Intuitively, these features are expected to correlate with the difficulty of the task.
In Figure \ref{fig:failure_modes} (a-c) we plot the success rate vs. the features described above. As expected, success rate decreases with plan length. Interestingly, however, several domains that required a long time for FD were `easy' for the learned policy, and had a high success rate. Further investigation revealed that these domains had large open areas, which are `hard' for planners to solve due to a large branching factor, but admit a simple policy. An example of one such domain is shown in Figure \ref{fig:failure_modes} (d-e). We also note that the number of walls had no visible effect on success rate -- it is the configuration of the walls that matters, and not their quantity.

\begin{figure*}[h]
  \vspace{-1em}
  \centering
  \begin{subfigure}[b]{0.3\linewidth}
    \includegraphics[width=\textwidth]{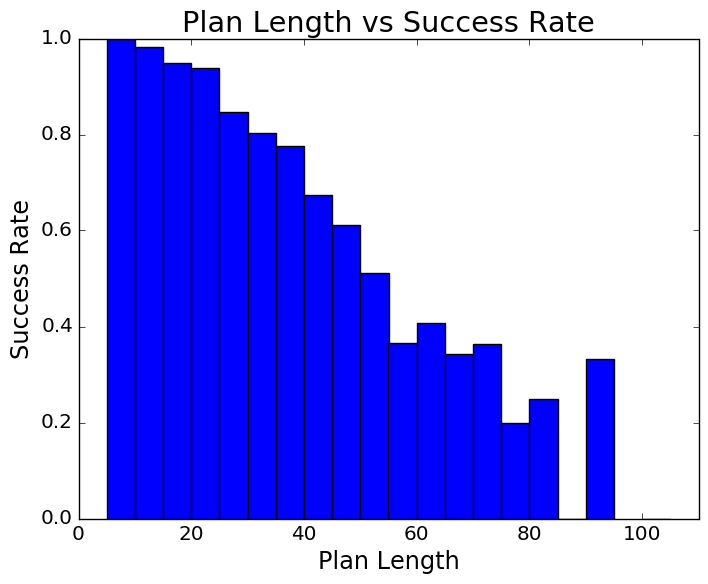}
    \caption{}
  \end{subfigure}~
  \begin{subfigure}[b]{0.3\linewidth}
    \includegraphics[width=\textwidth]{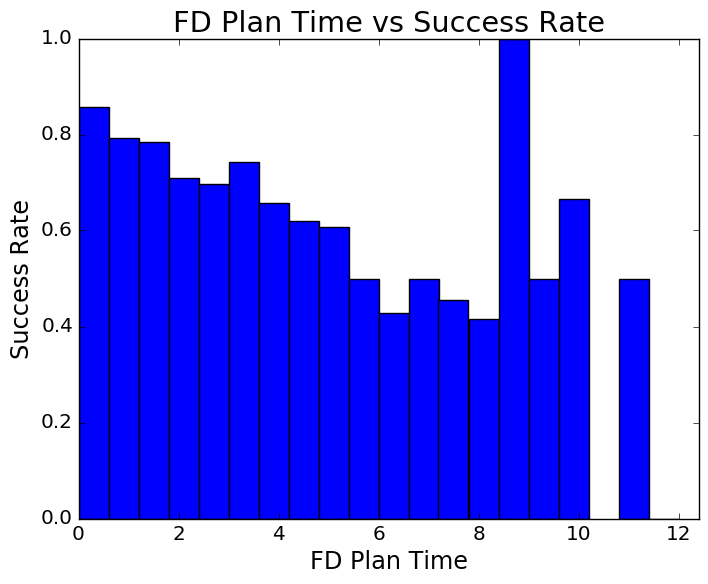}
    \caption{}
  \end{subfigure}~
  \begin{subfigure}[b]{0.3\linewidth}
    \includegraphics[width=\textwidth]{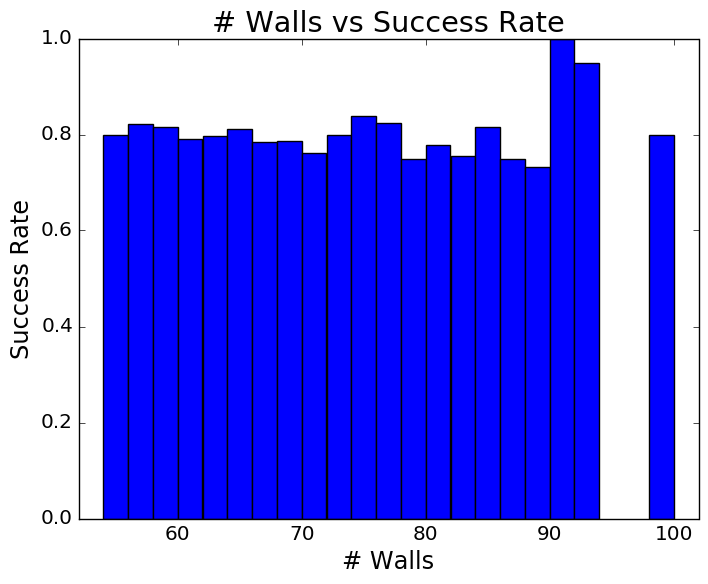}
    \caption{}
  \end{subfigure}\\
  \begin{subfigure}[b]{0.19\linewidth}
    \includegraphics[width=\textwidth]{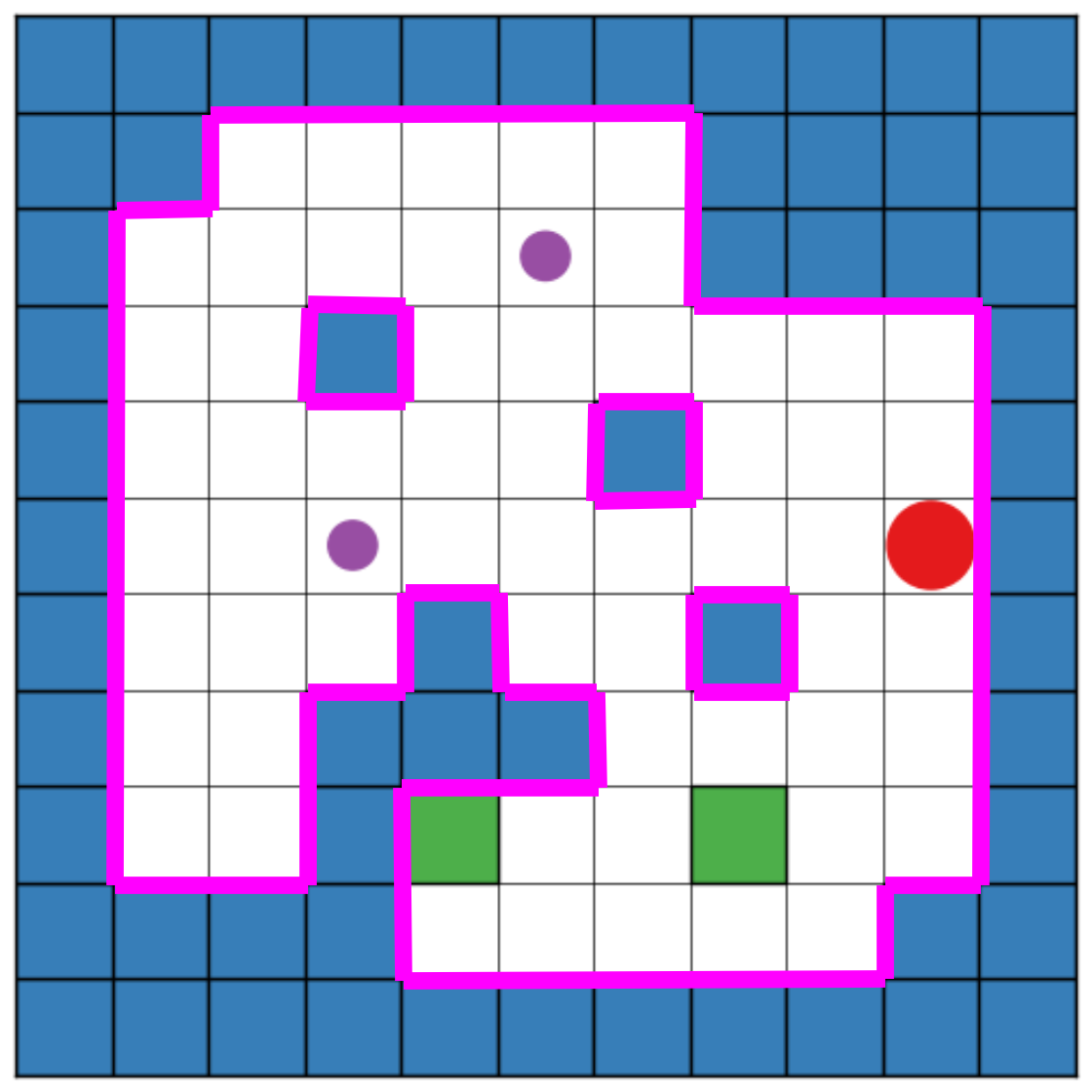}
    \caption{}
  \end{subfigure}~
  \begin{subfigure}[b]{0.19\linewidth}
    \includegraphics[width=\textwidth]{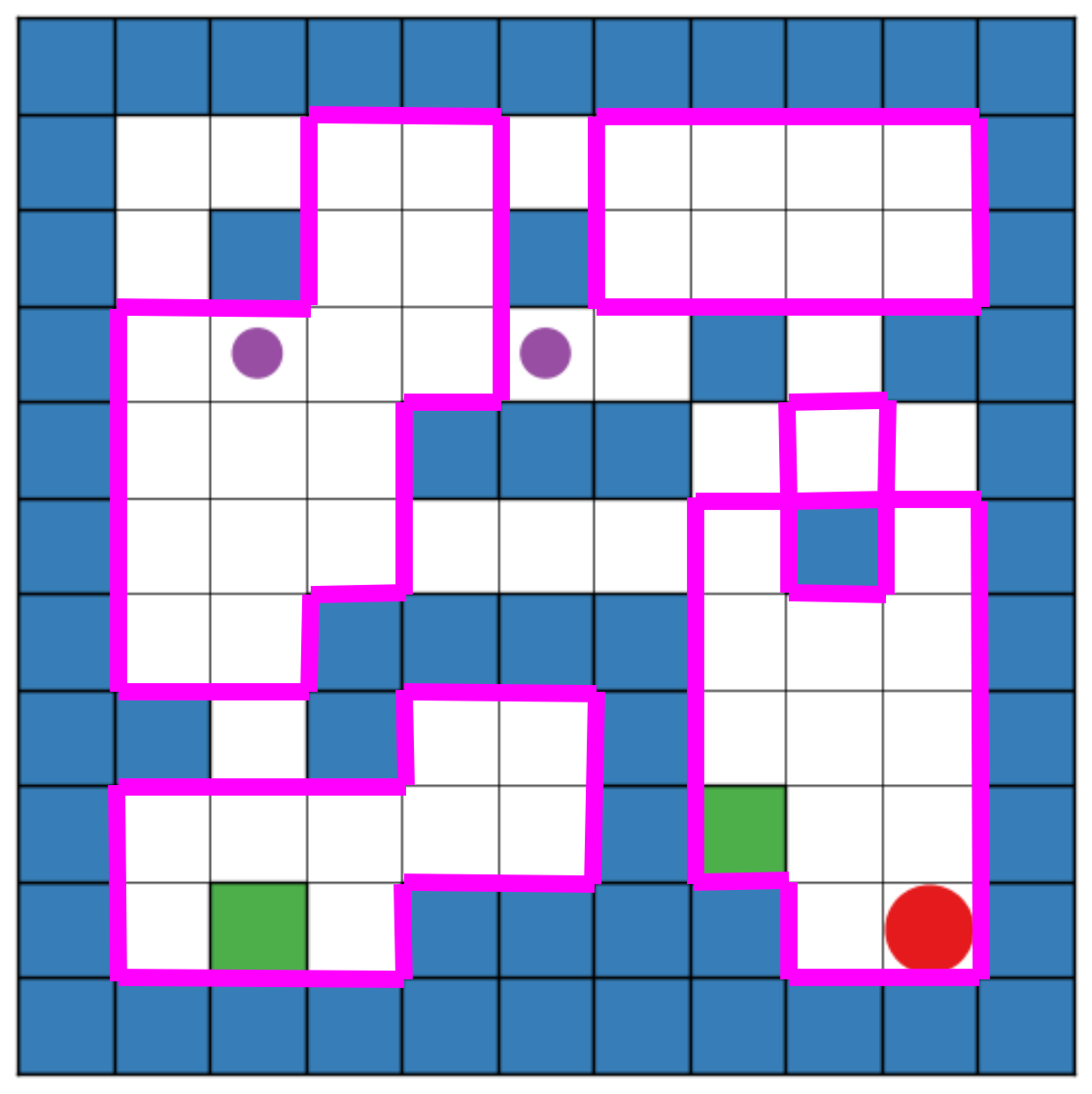}
    \caption{}
  \end{subfigure}~
  \begin{subfigure}[b]{0.19\linewidth}
    \includegraphics[width=\textwidth]{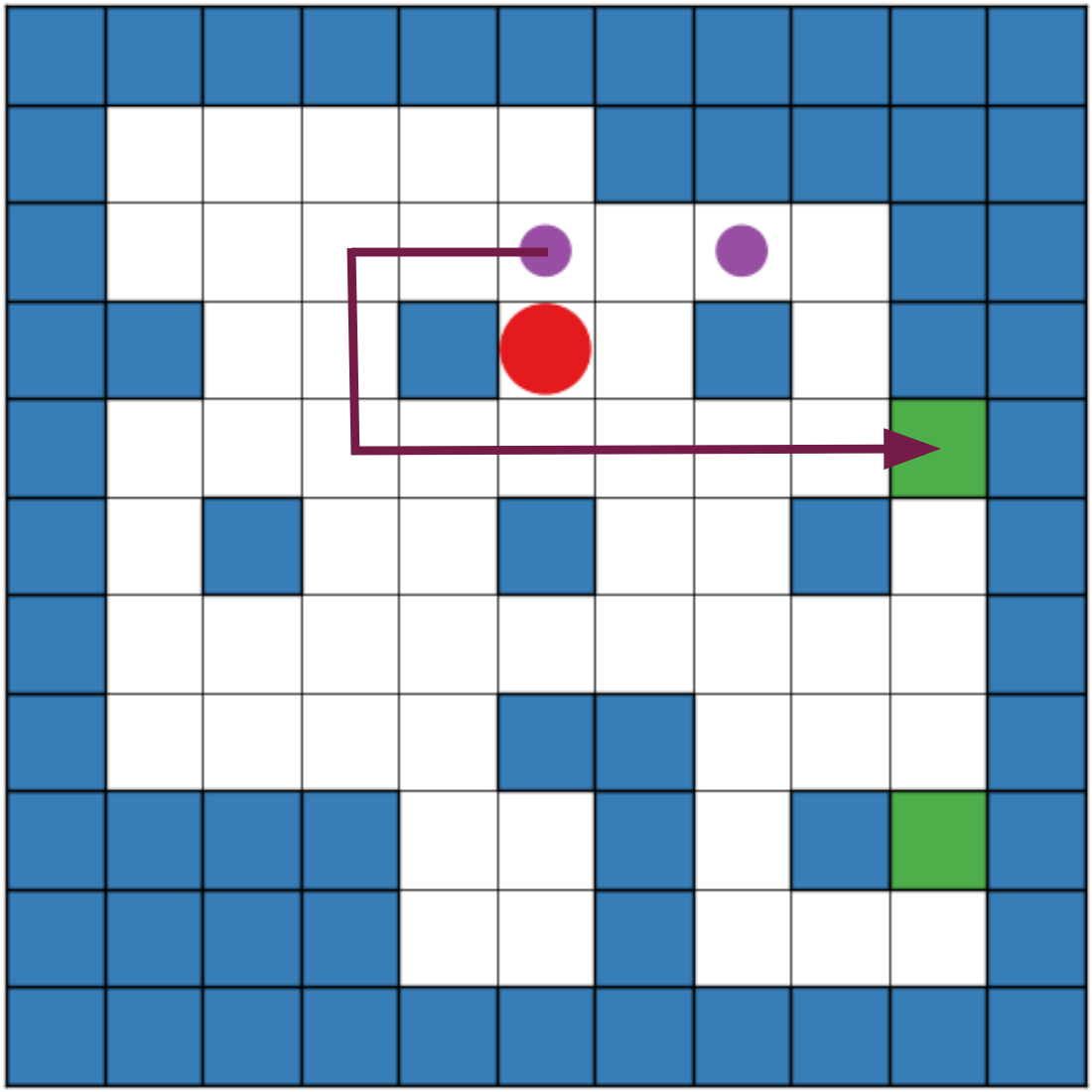}
    \caption{}
  \end{subfigure}~
  \begin{subfigure}[b]{0.19\linewidth}
    \includegraphics[width=\textwidth]{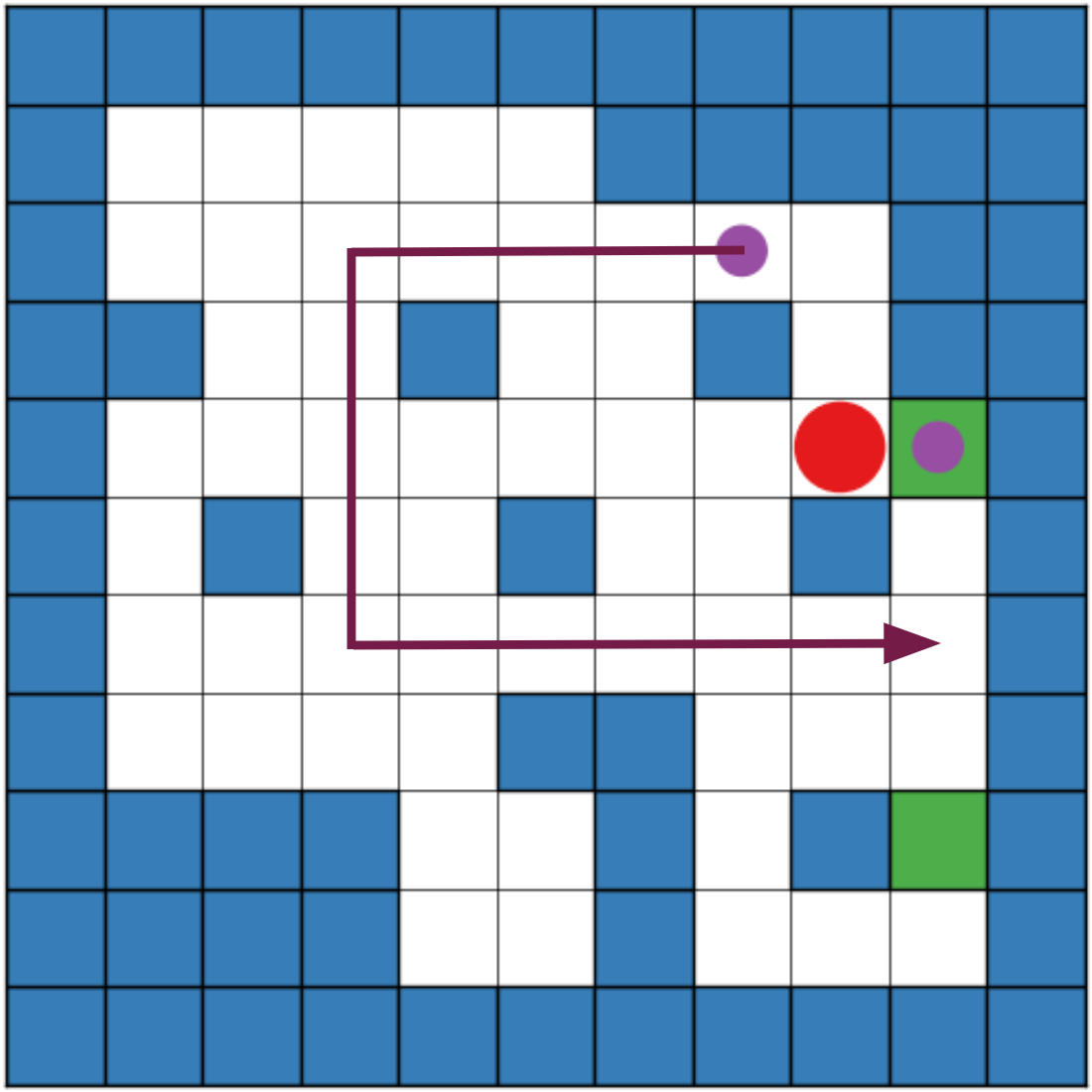}
    \caption{}
  \end{subfigure}~
  \begin{subfigure}[b]{0.19\linewidth}
    \includegraphics[width=\textwidth]{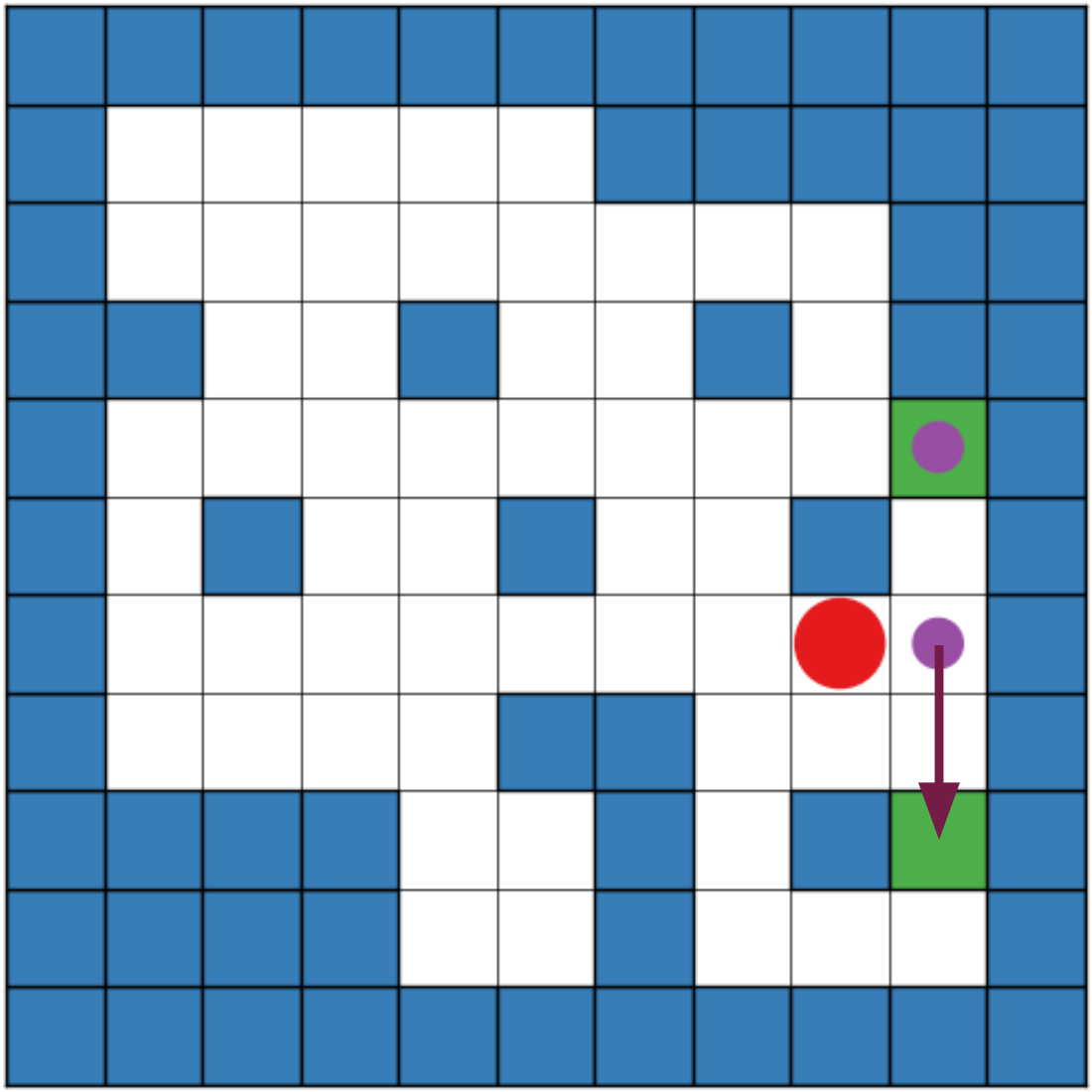}
    \caption{}
  \end{subfigure}
\caption{Analysis of Failure Modes. (a-c): Success rate vs features of the domain. Plan length (a) seems to be the main factor in determining success rate. Longer plans fail more often. While there is some relationship between planning time and success rate (b), planning time is not always an accurate indicator, as explained in (d,e). The number of walls (c) does not affect success rate. (d,e): Domains containing large open rooms results in a high branching factor and thus produce the illusion of difficulty while still having a simple underlying policy. The domain in (d) took FD significantly longer time to solve, ~8.6 seconds compared to ~1.6 seconds for the domain in (e), although it has a shorter optimal solution, 51 steps compared to 65 steps. This is since the domain in (e) can be broken up into small regions which are all connected by hallways, a configuration that reduces the branching factor and thus the overall planning speed. (f-h): Demonstration of the 2nd failure mode in Section \ref{ssec:failure_modes}. From the start state, the policy moves the first object using the path shown in (f). It proceeds to move the next object using the path in (g). As the game state approaches (h) it becomes clear that the current domain is no longer solvable. The lower object must be pushed down but is blocked by the upper object, which can no longer be moved out of the way. In order to solve this level, the first object must ether be moved to the bottom goal or must be moved after the second object has been placed at the bottom goal. Both solutions require a look-ahead consisting of 20+ steps.}
  \label{fig:failure_modes}
\end{figure*}

\subsection{Sokoban Reproducibility Details}
For all experiments we used a decaying learning rate $lr = lr_0(\frac{1}{2})^{\text{floor}(epoch/d)}$ where $lr_0=0.001$. We noticed that the decay rate was dependent on the amount of training data used. Less training data required a slower decay rate. When training with 45k trajectories we used $d=5$. The experiment for Figure \ref{fig:experiments}(a), \ref{fig:small_bootstrapping}, and for Table \ref{table:bootstrap_compare}, \ref{table:deep_vs_shallow} used a window size equivalent to the size of the world. The rest of the experiments used the window size $k=1$. The learning rate for Table \ref{table:bootstrap_compare} had $d=50$.

\subsection{TSP Reproducibility Details}
We use a decaying learning rate for all graph instances $lr = lr_0(.95)^{epoch}$ where $lr_0 = .001$. 
Our network architecture is 4-layers of the propagation rule defined in equation 2. Each layer consists of 26 neurons. We choose $\mathcal{N} = \mathcal{N}_1$, i.e. the network computes the convolution over all first-order neighbors of each node. We found that increasing this neighborhood size does not increase performance of the learned network. We include the edge weights of the graph as edge features. In our experiments, we use training sets of size 1,000 for all training and graph sizes.

\end{document}